\documentclass{article}

\PassOptionsToPackage{numbers, compress}{natbib}
\usepackage[preprint]{neurips_2020}

\usepackage[T1]{fontenc}    
\usepackage[hidelinks]{hyperref}       
\usepackage{url}            
\usepackage{booktabs}       
\usepackage{amsfonts}       
\usepackage{nicefrac}       
\usepackage{microtype}      

\usepackage[utf8]{inputenc}
\usepackage{amsmath}
\usepackage{amsthm}
\usepackage{amsfonts}
\usepackage{algorithm}
\usepackage{algorithmic}
\usepackage[dvipdfmx]{graphicx}
\usepackage{subfigure}

\newcommand{\E}{\mathop{{}\mathbb{E}}}
\renewcommand{\algorithmicrequire}{\textbf{Input:}}
\renewcommand{\algorithmicensure}{\textbf{Output:}}
\newtheorem{theorem}{Theorem}[section]

\newtheorem{proposition}[theorem]{Proposition}

\usepackage{color}
\usepackage{comment}

\newcommand{\argmin}{\mathop{\rm arg~min}\limits}
\renewcommand{\cite}{\citep}

\title{Regret Minimization for Causal Inference \\ on Large Treatment Space}

\author{%
  Akira Tanimoto\\
  NEC Corporation, Kyoto University, RIKEN AIP\\
  \texttt{a.tanimoto@nec.com}
  \And
  Tomoya Sakai\\
  NEC Corporation, RIKEN AIP\\
  \texttt{tomoya\_sakai@nec.com}\\
  \And Takashi Takenouchi\\
  Future University Hakodate, RIKEN AIP\\
  \texttt{ttakashi@fun.ac.jp}\\
  \And Hisashi Kashima\\
  Kyoto University, RIKEN AIP\\
  \texttt{kashima@i.kyoto-u.ac.jp}
}

\begin{document}

\maketitle

\begin{abstract}
    Predicting which action (treatment) will lead to a better outcome is a central task in decision support systems.
    To build a prediction model in real situations, learning from biased observational data is a critical issue due to the lack of randomized controlled trial (RCT) data.
    To handle such biased observational data, recent efforts in causal inference and counterfactual machine learning have focused on debiased estimation of the potential outcomes on a binary action space and the difference between them, namely, the individual treatment effect. 
    When it comes to a large action space (e.g., selecting an appropriate combination of medicines for a patient), however, the regression accuracy of the potential outcomes is no longer sufficient in practical terms to achieve a good decision-making performance.
    This is because the mean accuracy on the large action space does not guarantee the nonexistence of a single potential outcome misestimation that might mislead the whole decision. 
    Our proposed loss minimizes a classification error of whether or not the action is relatively good for the individual target among all feasible actions, which further improves the decision-making performance, as we prove.
    We also propose a network architecture and a regularizer that extracts a debiased representation not only from the individual feature but also from the biased action for better generalization in large action spaces.
    Extensive experiments on synthetic and semi-synthetic datasets demonstrate the superiority of our method for large combinatorial action spaces.

\end{abstract}

\section{Introduction}
Predicting individualized causal effects is an important issue in many domains for decision-making.
For example, a doctor considers which medication would be the most effective for a patient, a teacher considers which problems are most effective for improving the achievement of a student, and a retail store manager considers which assortment would improve the overall store sales.
To support such decision-making, we consider providing a prediction of which actions will lead to better outcomes.

Recent efforts in causal inference and counterfactual machine learning have focused on making predictions of the potential outcomes that correspond to each action for each individual target based on observational data.
Observational data consists of features of targets, past actions actually taken, and their outcomes.
We have no direct access to the past decision-makers' policies, i.e., the mechanism of how to choose an action under a target feature given.
Unlike in normal prediction problems, 
pursuing high-accuracy predictions only with respect to the historical data carries the risk of incorrect estimates due to the biases in the past policies.
These biases are also known as {\it spurious correlation} \cite{simon1954spurious, pearl2009causality}, which might mislead the decision-making. 
For those cases where real-world experiments such as randomized controlled trials (RCTs) or multi-armed bandit is infeasible or too expensive, causal inference methods provide debiased estimation of potential outcomes from observational data.

While most of the existing approaches assume limited action spaces such as binary ones as in individual treatment effect estimation (ITE), 
there are many real-world situations where the number of options is large.
For example, doctors need to consider which combination of medicines will best suit a patient.
For such cases, it is difficult to apply existing methods (as in \cite{shalit2017estimating,yoon2018ganite,schwab2018perfect}) for two reasons.
First, since the sample sizes for each action would be limited,
building models for each action (or a multi-head neural network), which existing methods adopt, is not sample-efficient.
Second, even if we manage to achieve the same level of regression accuracy as when the action space is limited, the same decision-making performance is no longer guaranteed in a large action space, as we prove in Section~\ref{sec:problem_relation_MAE_CG}.
This is because, in short, even though overestimation of the potential outcome for only a single action in many alternatives has only a small impact on the overall regression accuracy, it can mislead the whole decision to a bad action and result in a poor decision performance.

To achieve informative causal inference for decision-making in a large action space, we propose solutions for the above two problems.
For the sample-efficiency, we directly formulate the observational bias problem as a domain adaptation from a biased policy to a uniform random policy, which enables the extraction of debiased representations from both the individual features and the actions.
Thereby, we can build a debiased single-head model, aiming at better generalization for the large action space.
For the second issue, we analyze our defined decision-focused performance metric, ``regret'', and find that we can further improve the decision performance by minimizing the classification error of being in the top-$k$ best actions among feasible actions for each target, in addition to the regression error (MSE).
We cannot directly observe whether the action is in the top-$k$ since only one action and its outcome is observed for each target, and we propose a proxy loss that compares the observed outcome to the estimated conditional average performance of the past decision-makers.

In summary, our proposed method minimizes both the classification error and the MSE using debiased representations of both the features and the actions.
We demonstrate the effectiveness of our method through extensive experiments with synthetic and semi-synthetic datasets.

\begin{figure}[tb]
  \begin{center}
    \subfigure[Sample for ITE]{
        \includegraphics[keepaspectratio, scale=0.3]{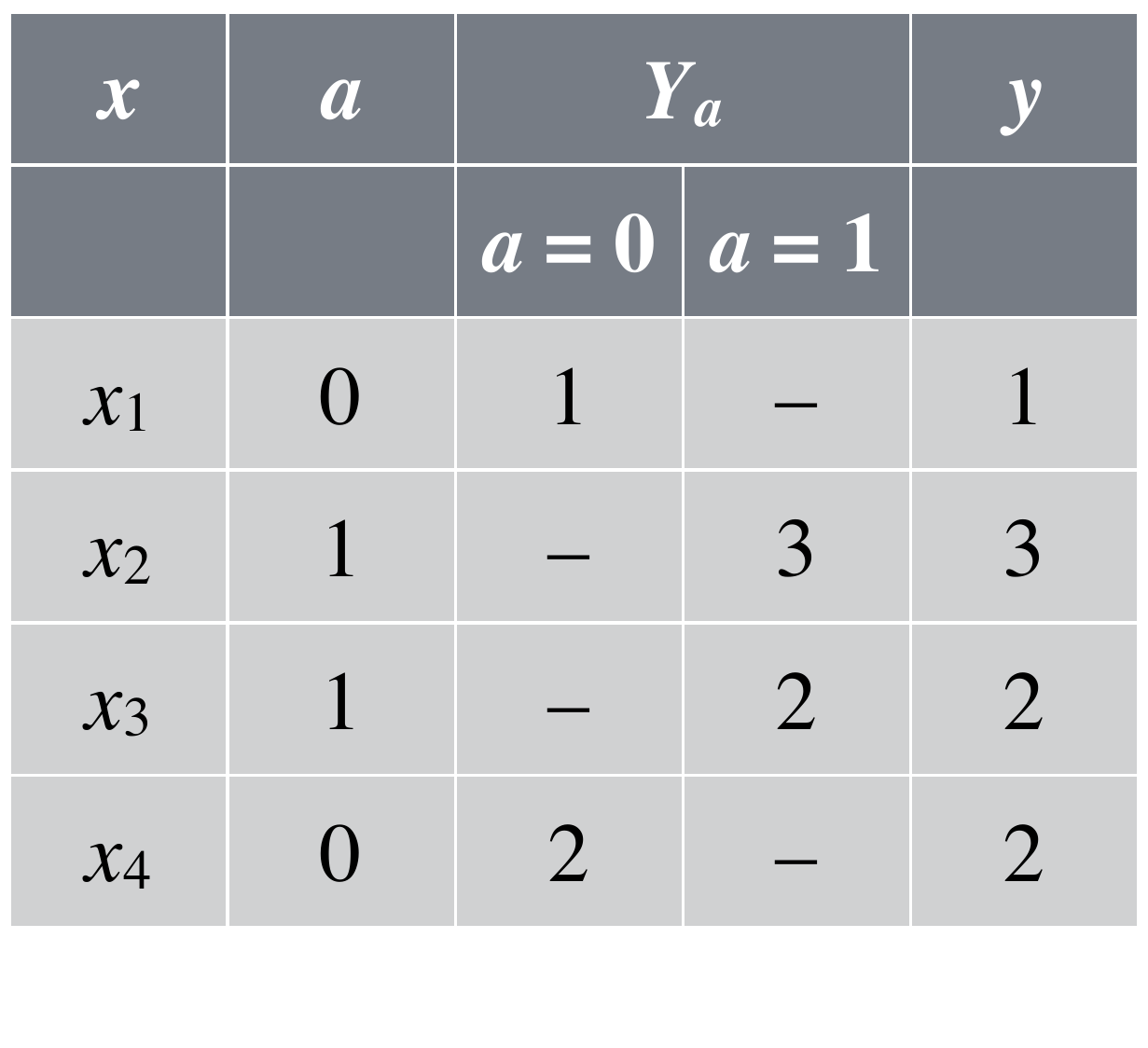}
        \label{fig:sample_table_ite}
    }
    \quad
    \subfigure[Sample for causal inference on a combinatorial action space]{
        \includegraphics[keepaspectratio, scale=0.3]{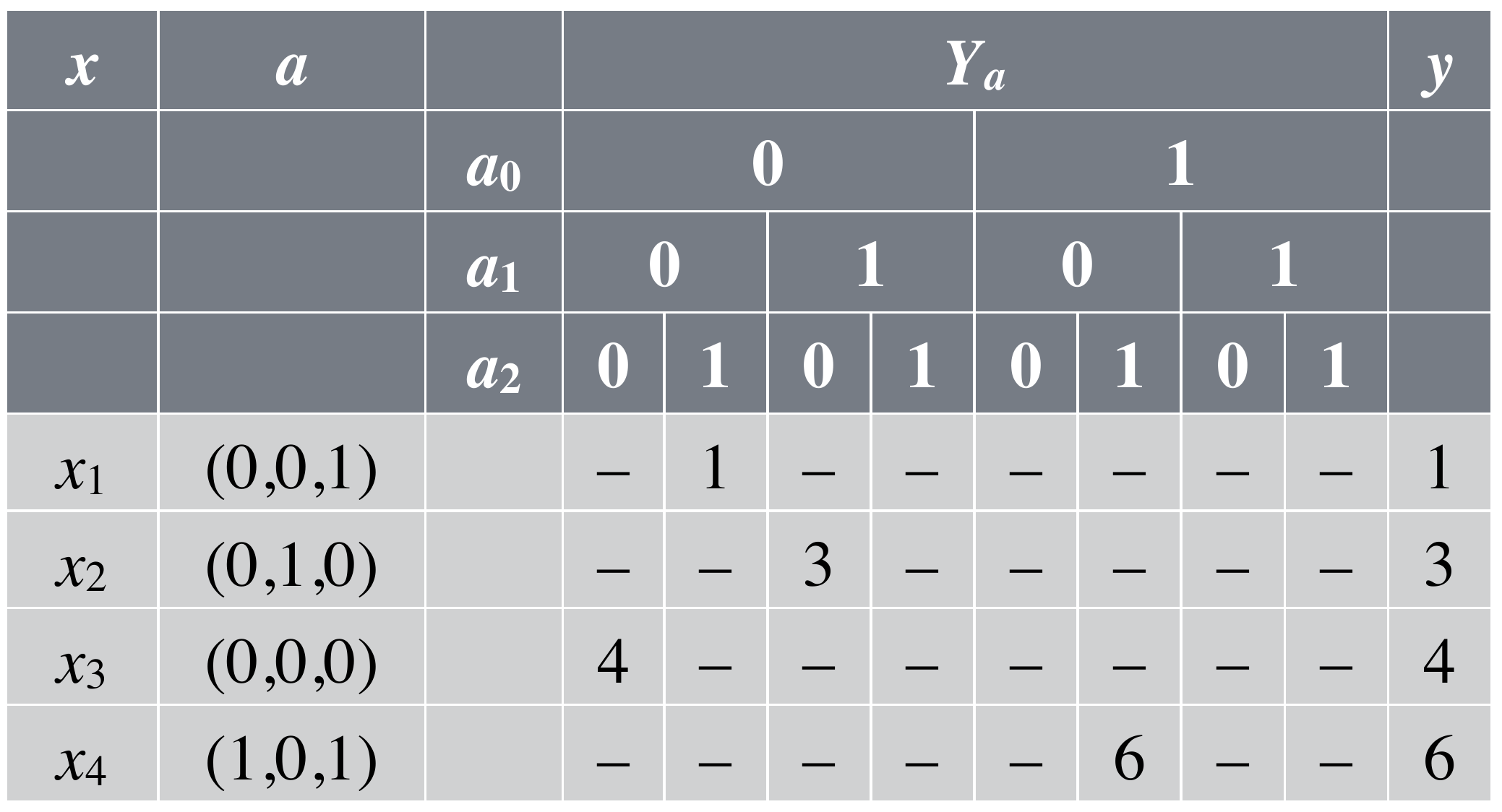}
        \label{fig:sample_table_combi_single}
    }
    \caption{Example data tables for the ITE estimation and our setting on a combinatorial action space. Dashes indicate missing entries. Both tasks can be seen as learning from missing-at-random data. Only factual outcomes are observed (when $a=a^\prime$, $y_{a^\prime}$ is observed) and the counterfactual records $\{y_a\}_{a\neq a^\prime}$ are missing.}
    \label{fig:sample_table}
  \end{center}
\end{figure}

\section{Problem setting}
\label{sec:problem}
In this section, we formulate our problem and define a decision-focused performance metric.
Our aim is to build a predictive model to inform decision-making.
Given a feature vector $x \in \mathcal X \subset \mathbb R^d,$ the learned predictive model is expected to correctly predict which action $a \in \mathcal{A}(x)$ leads to a better outcome $y \in \mathcal Y \subset \mathbb R$, where $\mathcal A(x)$ is a feasible subset of finite action space $\mathcal A$ given $x$.
We hereafter assume feasible action space does not depend on the feature, i.e., $\mathcal A(x)=\mathcal A$, for simplicity.
As a typical case of large action spaces, we assume an action consisting of multiple causes, i.e., $\mathcal A = \{0,1\}^m$ (combinatorial action space).

We assume there exists a joint distribution
$p(x,a,y_1,\ldots,y_{|\mathcal A|}) = p(x)\mu(a|x)p(y_1,\ldots,y_{|\mathcal A|}|x),$ where $\mu(a|x)$ is the unknown decision-making policy of past decision-makers, called propensity, and $y_1,\ldots,y_{|\mathcal A|}$ are the potential outcomes corresponding to each action.
The observed (factual) outcome $y$ is the one corresponding to the observed action $a$, i.e., a training instance is $(x_n,a_n,y_{a_n})$, where $n$ denotes the instance index, and the other (counterfactual) potential outcomes are regarded as missing as shown in Fig.~\ref{fig:sample_table}. 
Note that the joint distribution is assumed to have conditional independence $(y_1, \ldots, y_{|\mathcal A|})\perp a | x$ (unconfoundedness). 
In addition, we assume $\forall a \in \mathcal{A}$ and $\forall x$, $0<\mu(a | x)<1$ (overlap). 
These are commonly required to identify causal effects \cite{imbens2009recent, pearl2009causality}.

To define a performance measure of a model, we utilize a simple prediction-based decision-making policy: 
given a parameter $k \in [|\mathcal A|]$, 
$$
  \pi_k^f(a|x) := \left\{
    \begin{aligned}
    1 / k &~~ \text { if } \mathrm{rank}(f(x,a); \{f(x,a^\prime)\}_{a^\prime}) \le k \\
    0 &~~ \text { otherwise, }
    \end{aligned}
    \right.
$$
where $\mathrm{rank}(\cdot)$ denotes the rank among all the feasible actions $\mathcal{A}$, i.e., $\mathrm{rank}(f(x,a);\{f(x,a^\prime)\}_{a^\prime}):=\left|\{a^\prime \mid f(x,a^\prime) \ge f(x,a), a^\prime \in \mathcal{A}\}\right|$.
We also denote as $\mathrm{rank}(f(x,a))$ for short.
This means choosing an action uniformly at random from the predicted top-$k$ actions.

Here we define the performance of $\pi_k^f$ by its expected outcome $\E_{p(x)\pi_k^f(a|x)}[y_a]$,
which can be written as the following mean cumulative gain (mCG), and we also define its difference from the oracle's performance (regret):
\begin{align}
  \mathrm{mCG}@k(f) &:= \frac{1}{k}\E_x \left[\sum_{a : \mathrm{rank}(f(x,a)) \le k} \bar{y}_{a} \right], \label{eq:def_cg} \\
  \mathrm{Regret}@k(f) &:= \frac{1}{k}\E_x \left[\sum_{a : \mathrm{rank}(\bar y_a) \le k} \bar{y}_{a}\right] - \mathrm{mCG}@k(f) , \label{eq:def_regret}
\end{align}
where $\bar{y}_a := \E[y_a|x]$ is the expected potential outcome and $\mathrm{rank}(\bar y_a)$ is its rank among all the feasible actions.
Here $\left(1-\mathrm{mCG}@1(f)\right)$ is known as the policy risk \cite{shalit2017estimating}.
Since the first term in (\ref{eq:def_regret}) is constant with respect to $f$, the mCG and the regret are two sides of the same coin as the performance metrics of a model.
We regard the mCG (or the regret) as the metric in this paper.

\section{Relation between prediction accuracy and precision in decision-making}
\label{sec:problem_relation_MAE_CG}
In this section, we analyze our decision-focused performance metric $\mathrm{Regret}@k$.
Our analysis reveals the difficulty of causal inference in a large action space that the regret bound get worse for the same regression accuracy.
At the same time, however, it is shown that we can improve the bound by simultaneously minimizing a classification error, which leads to our proposed method.

A typical performance measure in existing causal inference studies that is applicable to large action spaces is the following uniform MSE \cite{schwab2018perfect,yoon2018ganite}:
\begin{align}
  \label{eq:unbiased mse}
  \mathrm{MSE}^u(f) = \E_x \left[ \frac{1}{\left| \mathcal{A}\right|} \sum_{a \in \mathcal{A}} \E_{y_a|x} [(y_a - f(x,a))^2] \right].
\end{align}
Note that $\mathrm{MSE}^u$ is different from the normal MSE in the supervised machine learning context in which the expectation is taken over the same distribution as the training, i.e., $\mathrm{MSE} = \E_{y, a, x \sim p(y_a|x)\mu(a|x)p(x)} \left[ (y-f(x,a))^2 \right]$.
We refer to $\mathrm{MSE}^u$ as MSE, or specifically the uniform MSE, in this paper.

Here the relation between the uniform MSE and the regret is the following (proof is in Appendix~\ref{sec:proof_regret_bound}).
\begin{proposition}
  \label{thm:relation_cg_mse}
  The regret in \eqref{eq:def_regret} will be bounded 
  with uniform MSE in \eqref{eq:unbiased mse} as
  \begin{align}
  \label{eq:regret-upper-bound}
    \mathrm{Regret}@k(f) \le \frac{|\mathcal A|}{k} \sqrt{\mathrm{ER}_k^u(f) \cdot \mathrm{MSE}^u(f)},
  \end{align}
  where $\mathrm{ER}_k^u(f)$ is 
  the top-$k$ classification error rate, i.e.,
  \begin{align*}
    \mathrm{ER}_k^u(f) := \E_x \left[ \frac{1}{|\mathcal{A}|} \sum_{a \in \mathcal{A}} I\left((\mathrm{rank}(y_{a})\le k) \oplus (\mathrm{rank}(f(x,a))\le k)\right) \right],
  \end{align*}
  where $\oplus$ denotes the logical XOR.
\end{proposition}

Since $\mathrm{ER}_k^u(f) \le 1$ for any $f$, we see that only minimizing the uniform MSE as in existing causal inference methods leads to minimizing the regret.
However, when $|\mathcal{A}|/k$ is large, the bound would be loose, and only unrealistically small $\mathrm{MSE}^u$ provide a meaningful guarantee for the regret.
At the same time, we see that the bound can be further improved by minimizing the uniform top-$k$ classification error rate $\mathrm{ER}_k^u(f)$ simultaneously, which leads to our proposed method.

\section{Regret minimization network: debiased potential outcome regression and classification on a large action space}
\label{sec:method}
Our proposed method, regret minimization network (RMNet), consists of two parts.
First, we introduce our loss that aim to minimize the regret by minimizing both 
$\mathrm{MSE}^u$ and $\mathrm{ER}^u_k$.
Then, we introduce a sample-efficient network architecture, in which a representation is extracted from both the feature $x$ and the action $a$, and a representation-based debiasing regularizer that performs domain adaptation according to the structure.

\subsection{Uniform regret minimization loss}
As we saw in Section~\ref{sec:problem_relation_MAE_CG}, 
we can improve the decision-making performance by minimizing the uniform top-$k$ classification error rate
$\mathrm{ER}_k^u$.
Notice that the r.h.s. of Eq.~\eqref{eq:regret-upper-bound} is 
bounded as follows:
$\sqrt{\mathrm{ER}_k^u(f) \cdot \mathrm{MSE}^u(f)} \le \left(\gamma  \mathrm{ER}_k^u(f) + \mathrm{MSE}^u(f) / \gamma \right)/2$
from the inequality of arithmetic and geometric means, for $\gamma>0$, 
and the equality holds when $\gamma = \sqrt{ \mathrm{MSE}^u(f)\big/ \mathrm{ER}_k^u(f)}$.
We thus aim to minimize the weighted sum of $\mathrm{ER}_k^u$ and $\mathrm{MSE}^u_k$.

Since we observe only one action and its outcome for each target,
we cannot directly estimate $\mathrm{ER}_k^u(f)$, which is based on the ranked list of potential outcomes, only from the data.
Therefore, we recast the minimization of $\mathrm{ER}_k^u(f)$ into a simple classification.

First, we rewrite $\mathrm{ER}_k^u(f)$ with the 0-1 classification risk as follows
(the derivation is in Appendix~\ref{sec:der-ranking-to-clf}):
\begin{align}
  \mathrm{ER}_k^u(f)
    = \E_{x}\left[ \frac{1}{|\mathcal{A}|}\sum_{a\in\mathcal{A}}
  \ell_{0-1}(y_a - y_{a_k^\ast}, f^\prime(x,a) - y_{a_k^\ast}) \right],
  \label{eq:ranking-to-clf}
\end{align}
where $f^\prime(x,a) = f(x,a)-f(x,\hat a_k^\ast) + y_{a_k^\ast}$ and $\ell_{0-1}(t,\hat t):=I(t\ge 0 \oplus \hat t \ge 0)$ is the 0-1 classification loss. 
Here the terms $-f(x,\hat a_k^\ast) + y_{a_k^\ast}$ are constant with respect to $a$,
and $\mathrm{ER}_k^u(f)=\mathrm{ER}_k^u(f^\prime)$ thus holds.
Therefore, we optimize the 0-1 loss with respect to $f^\prime$.

Next, we replace the unobservable $k$-th best outcome $y_{a_k^\ast}$ in \eqref{eq:ranking-to-clf} with the conditional average outcome $\E_{a\sim \mu(a|x)}[y_a|x]$, which can be estimated by a model trained using observational data as $g=\argmin_{g^\prime} \frac{1}{N} \sum_n (y_n-g^\prime(x_n))^2$.
This means that we do not optimize $\mathrm{ER}_k^u(f)$ for arbitrary $k$ but for a specific $k$ that corresponds to the average performance of the observational policy, i.e., such $k$ that satisfies $y_{a_{k+1}^\ast} \le \E_{a \sim \mu(a|x)}[y_a|x] \le y_{a_k^\ast}$ ($k$ may depend on $x$).
The replaced numerical label $y-g(x)$ is called residual\footnote{Also known as the advantage in the reinforcement learning \cite{mnih2016asynchronous}.}.
A positive residual $y-g(x)>0$ means that the action $a$ outperformed the conditional average performance of the observational policy, thus ranking such $a$ higher under $x$ leads to superior performance to the past decision-makers.

Considering the noise on the residual $y-g(x)$ due to the noise on $y_a$ and the estimation error of $g(x)$, we train our model with an estimation of the true label called \emph{soft-label} \citep{peng2014learning,nguyen2011learning} $\sigma(y-g(x)),$ 
where $\sigma(t):=1/(1+\exp(-t))$ is the sigmoid function, instead of a naive plug-in label $I(y-g(x))$.
The proposed proxy risk for $\mathrm{ER}_k^u$ is the following cross-entropy:
\begin{align}
  \label{eq:soft cross entropy}
  L_\mathrm{cl}^u(f;g) = - \E_{x}\left[ \frac{1}{|\mathcal{A}|}\sum_{a\in\mathcal{A}} \{  s \log v + (1-s)\log(1-v) \} \right],  
\end{align}
where $s:=\sigma(y - g(x))$ and $v:=\sigma(f(x,a) - g(x))$.
Note that the loss for each $n$ is minimized when $y_n = f(x_n, a_n)$ regardless of $g(x_n)$, as illustrated in Fig.~\ref{fig:loss_shape} in the appendix.

After all, our risk is defiend as 
the weighted sum of the classification risk and the MSE:
\begin{align}
  \label{eq:uniform_expected_risk}
  L^u(f;g,\beta) = \beta L_\mathrm{cl}^u(f;g) + (1-\beta) \mathrm{MSE}^u(f),
\end{align}
where $0\leq\beta\leq1$.

\subsection{Debiasing by representation-based domain adaptation to the RCT policy}
\label{subsec:domain adaptation}
While accessible observational data is biased by the propensity $\mu(a|x)$,
the expected risk $L^u(f;g)$ is averaged over all actions uniformly.
In this section, therefore, we construct a debiased empirical risk against the sampling bias.
Also, we propose an architecture that extracts representations from both the feature and the action for better generalization in a large action space.

There are two major approaches for debiased learning in individual-level causal inference.
One is density estimation-based method called inverse probability weighting using propensity score (IPW) \cite{austin2011introduction}, in which each instance is weighted by $1/\mu(a_n|x_n)$.
Since the expected risk matches the one of RCT, a good performance can be expected asymptotically under accurate estimation of $\mu$ or when it is recorded as in logged bandit problems.
However, in observational studies, where the propensity has to be estimated and plugged-in, its efficacy would easily drop \cite{kang2007demystifying}.
The other approach is representation balancing \cite{shalit2017estimating,johansson2016learning}, in which a model consists of representation extractor $\phi$ and hypotheses $\{h_a\}_a$ as in Fig.~\ref{fig:structure_cfr} and the conditional probabilities of representations $\{p(\phi|a)\}_a$ are encouraged to be similar to each other by means of so-called integral probability metric (IPM) regularizer.
We also take this approach and extend for large action spaces.

It is difficult to naively extend these methods to a large action space.
A reason is, as in Fig.~\ref{fig:structure_cfr}, constructing hypothesis layers for each action is not sample-efficient.
Also, representation balancing of each pair of actions $D_\mathrm{IPM}(p_a(\phi), p_{a^\prime}(\phi))$ is computationally and statistically infeasible.
Therefore, we propose extracting representations from both the features and the action as in Fig.~\ref{fig:general_structure}.

\begin{figure}[tb]
  \centering
  \subfigure[Counterfactual Regression (CFR)]{
 \includegraphics[keepaspectratio,width=60mm]{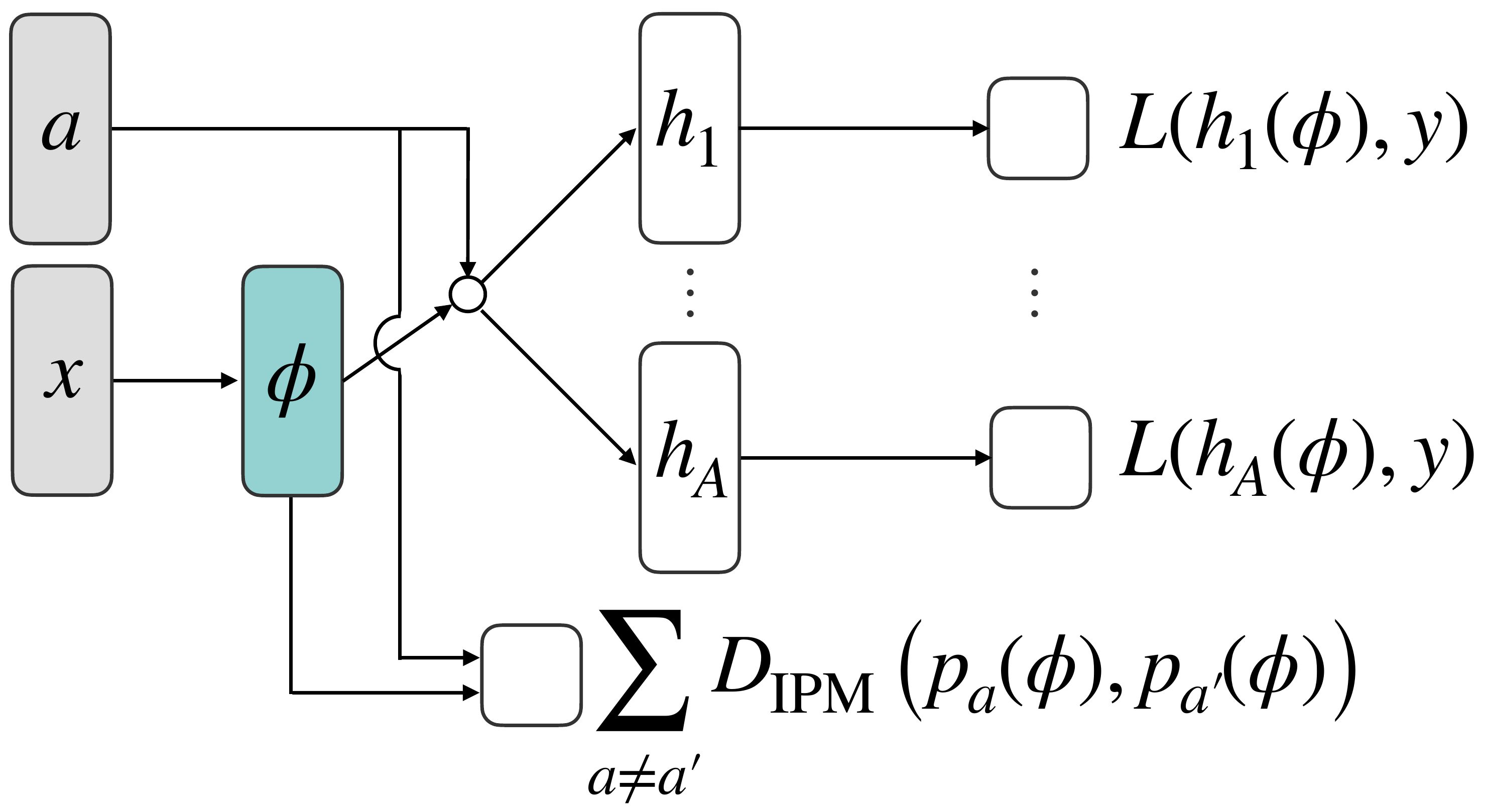}
\label{fig:structure_cfr}
  }
  \subfigure[Proposed architecture]{
  \centering
\includegraphics[keepaspectratio,width=60mm]{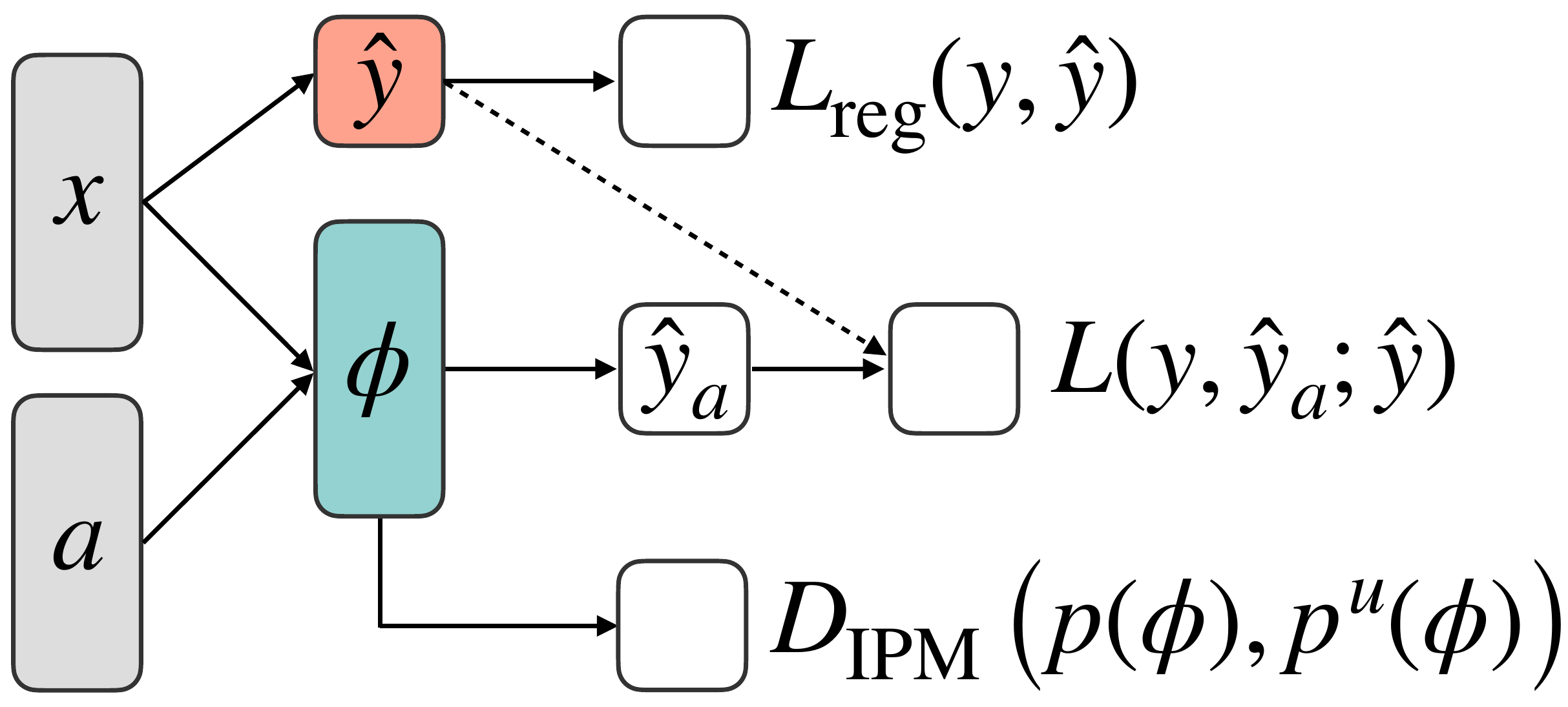}
\label{fig:general_structure}
  }
  \label{fig:structure}
  \caption{Network structures of Counterfactual regression for ITE \cite{shalit2017estimating,schwab2018perfect}
   \protect\footnote{As for the representation balancing regularizer $\sum_{a\neq a^\prime}D_{\mathrm{IPM}}$, \cite{shalit2017estimating} assumed $|\mathcal{A}|=2$ and \cite{schwab2018perfect} extended it to $|\mathcal{A}|\ge 2$, but both assumed the existence of a special action $a_0$ (e.g., no intervention), and only the distances between $ a_0 $ and other actions are taken into account. We assume no such special action, and thus the pairwise comparison is a reasonable extension.}
   and our proposed method.
   A broken line indicates no backpropagation.
   }
\end{figure}

We want to minimize the risk under the joint distribution with the uniform policy $p^u(x,a)=p(x)\mathrm{Unif}(\mathcal A),$ where $\mathrm{Unif}(\mathcal A) = 1/|\mathcal A|$ denotes the discrete uniform distribution, using sample from observational joint distribution $p(x,a)=p(x)\mu(a|x)$.
This can be seen as an unsupervised domain adaptation task from the training distribution $p(x,a)$ to the joint distribution with the uniform policy $p^u(x,a)$.
From this observation, we directly apply the representation regularizer to these distributions.
That is, we encourage matching 
$p(\phi_{x,a}):= \int \sum_{a^\prime} p(\phi_{x,a}|x^\prime,a^\prime)\mu(a^\prime|x^\prime)p(x^\prime) \mathrm{d}x^\prime$
and
$p^u(\phi_{x,a}):= \int \sum_{a^\prime} p(\phi_{x,a}|x^\prime,a^\prime) p^u(x^\prime,a^\prime) \mathrm{d}x^\prime,$
where $p(\phi_{x,a}|x^\prime,a^\prime)=\delta(\phi_{x,a}-\phi(x^\prime,a^\prime))$.

The resulting objective function is
\begin{align}
\begin{split}
\label{eq:objective}
	\min_{f} &\frac{1}{N} \sum_n L(f(x_n, a_n), y_n; g(x_n), \beta)
     + \alpha \cdot D_\mathrm{IPM} \left( \{\phi(x_n,a_n)\}_n, \{\phi(x_n,a^u_n)\}_n \right)
     + \mathfrak{R}(f),
\end{split}
\end{align}
where $L$ is the empirical instance-wise version of (\ref{eq:uniform_expected_risk}), $a^u_n$ is sampled from $p^u(a|x=x_n)$, and $\mathfrak{R}$ is a regularizer.
We utilize the Wasserstein distance, which is an instance of the IPM, as the discrepancy measure of the representation distributions, as in \cite{shalit2017estimating}.
Specifically, we use an entropy relaxation of the exact Wasserstein distance, called Sinkhorn distance \cite{cuturi2013sinkhorn}, for the compatibility with the gradient-based optimization.
The resulting learning flow is shown in Algorithm~\ref{alg1}.
A theoretical analysis for our representation balancing regularization can be found in Appendix~\ref{sec:theory}.

\begin{algorithm}[tb]
  \caption{Unbiased learning of outcomes}
  \label{alg1}
  \begin{algorithmic}[1]
    \REQUIRE Observational data $D = \{(x_n, a_n, y_n)\}_n$, hyperparameters $\alpha$ and $\beta$
    \ENSURE Trained network parameter $W$
    \STATE Train $g$ by an arbitrary supervised learning method with $D^\prime = \{(x_n, y_n)\}_n$, e.g.: \\
    $g = \argmin_{g^\prime} \sum (y_n-g^\prime(x_n))^2.$
  	\WHILE{Convergence criteria is not met}
    \STATE Sample mini-batch $\{n_1,\ldots,n_b\} \subset \{1,\ldots,N\}$.
    \STATE Calculate the gradient of the supervised loss $L$ in (\ref{eq:objective}): \\
    $g_1 = \nabla_{W} \frac{1}{b} \sum L(f(x_{n_i},a_{n_i};W), y_{n_i}; g(x_{n_i}), \beta).$ \\
    \STATE Sample uniformly random action \\ $\{a^u_1,\ldots,a^u_b\} \sim  \mathcal{A}^{b}$.
    \STATE Calculate the gradient of the representation balancing regularizer (e.g., Sinkhorn distance \cite{cuturi2013sinkhorn}): \\
    $g_2 = \nabla_{W} D_\mathrm{IPM}(\{\phi(x_{n_i},a_{n_i};W)\}, \{\phi(x_{n_i},a^u_{i};W)\}).$ \\
    \STATE Obtain step size $\eta$ with an optimizer (e.g., Adam \cite{kingma2015adam})
    \STATE $W \leftarrow [W - \eta(g_1+ \alpha  g_2)].$
    \STATE Check convergence criterion

    \ENDWHILE
    \STATE {\bf return} $W$
  \end{algorithmic}
\end{algorithm}

\section{Experiments}
We investigated the effectiveness of our method through synthetic and semi-synthetic experiments.
Both datasets were newly designed by us for the problem setting with a large action space.

\subsection{Experimental setup}

{\bf Compared methods.}
We compared our proposed method (RMNet) with ridge linear regression (OLS), random forests \cite{breiman2001random}, Bayesian additive regression trees (BART) \cite{hill2011bayesian}, naive deep neural network (S-DNN), naive DNN with multi-head architecture for each actions (M-DNN) (a.k.a. TARNET \cite{shalit2017estimating}), and straightforward extensions of the existing action-wise representation balancing method (counterfactual regression network (CFRNet)) \cite{shalit2017estimating}. 
We also made comparisons with the methods in which each one component of our proposed method was removed from the loss function, i.e., $\mathrm{MSE}$ (``w/o MSE''), $L_{cl}$ (``w/o ER''), and $D_\mathrm{IPM}$ (``w/o $D_\mathrm{IPM}$''), to clarify the contributions of each component.
For the main proposed method (RMNet), we equally weighted ER and MSE ($\beta=0.5$).
The strength of representation balancing regularizer $\alpha$ in CFRNet and proposed method was selected from $[0.1,0.3,1.0,3.0,10.0]$.
Other specification of DNN parameters can be found in Appendix~\ref{sec: experimental details}.

{\bf Evaluation.}
We used the normalized mean CG (nmCG) as the main metric, defined as follows.
\begin{align*}
  \mathrm{Normalized~CG}@k := \sum_{x,a:\mathrm{rank}(f(x,a))\le k} y_a(x) \Bigg/ \sum_{x,a:\mathrm{rank}(y_a(x))\le k} y_a(x).
\end{align*}
The normalized mean CG is proportional to the mean CG (\ref{eq:def_cg}) except that the expected outcomes are replaced with the actual ones.
We can see $\mathrm{Normalized~CG}@k \le 1$ from the definition of $\mathrm{rank}(y_a(x))$.
Since we have standardized the outcome, the chance rate is $\mathrm{Normalized~CG}@k=0.$
In addition to nmCG, we have also evaluated with respect to the uniform MSE.
The validation and the model selection was based on the mean CG, including the results in MSE.

{\bf Infrastructure.}
All the experiments were run on a machine with 28 CPUs (Intel(R) Xeon(R) CPU E5-2680 v4 @ 2.40GHz), 250GB memory, and 8 GPUs.

\subsection{Synthetic experiment}
{\bf Dataset.}
We prepared seven biased datasets in total to examine the robustness of the proposed and baseline methods. 
For detailed generation process, see Appendix~\ref{sec: experimental details}.
The feature space and the action space are fixed to $\mathbb R^5$ and $\{0,1\}^5$, respectively.
The sample sizes for $x$ were 1,000 for training, 100 for validation, 200 for testing.
For training, only one actions and the corresponding outcomes are sampled as follows.
Six of the settings have generalized linear models $y = f(x_\Upsilon,a_\Upsilon) + \varepsilon$, where $\cdot_\Upsilon$ denotes one-dimensional representations of $x$ and $a$.
The function $f$ is linear in three of them ($f(x_\Upsilon,a_\Upsilon) = a_\Upsilon - 2 x_\Upsilon$) and quadratic with respect to $a_\Upsilon$ in the rest ($f(x_\Upsilon,a_\Upsilon) = a_\Upsilon^2 - 2 x_\Upsilon$).
The last setting is a bilinear model $y = x^\top W a + \varepsilon.$
We set sampling biases 
as $p(a|x) \propto \exp(10\left|x_\Sigma - a_\Sigma\right|),$ where $\cdot_\Sigma$ denotes another representations of $x$ and $a$.
The three settings for linear and quadratic patterns correspond to the relation between $\cdot_\Sigma$ and $\cdot_\Upsilon$ as illustrated in Fig.~\ref{fig:graph_synthetic_data_simple}--\subref{fig:graph_synthetic_data_ps}, i.e.,
$x_\Sigma=x_\Upsilon$ ($=: x_\Delta$) in Setup-A and C, and $a_\Sigma=a_\Upsilon$ ($=: a_\Delta$) in Setup-B and C.
These relations of variables were designed to reproduce spurious correlation, which may mislead the deicision-making as follows.
In Setup-A, $a_\Sigma$ would have dependence to $y$ through the dependence to $x_\Delta$ despite $a_\Sigma$ itself has no causal relationship to $y$.
Samely, in Setup-B, $x_\Sigma$ would have dependence to $y$ through $a_\Delta$, and the causal effect of $a_\Delta$ may appear discounted.
In Setup-C, the causal effect of $a_\Delta$ might appear to be the opposite, as illustrated in Fig.~\ref{fig:illutration_linear_C}.

\begin{figure}[tb]
  \centering
  \centering
  \subfigure[Setup-A]{
    \includegraphics[keepaspectratio, width=20mm]{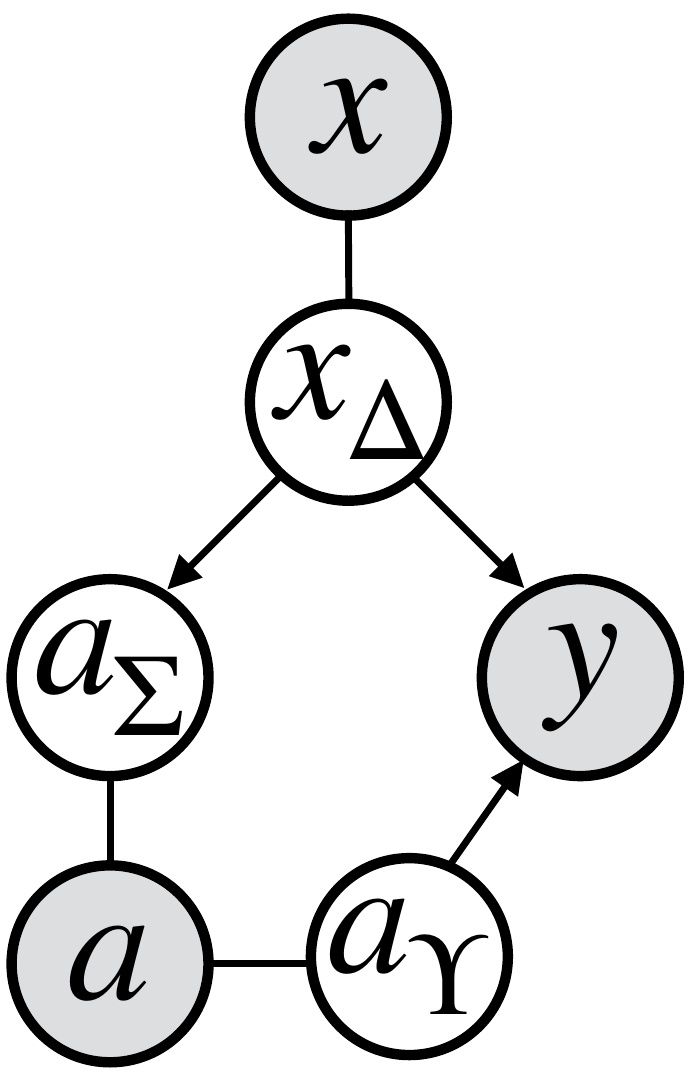}
    \label{fig:graph_synthetic_data_simple}
  }
  \subfigure[Setup-B]{
    \includegraphics[keepaspectratio, width=20mm]{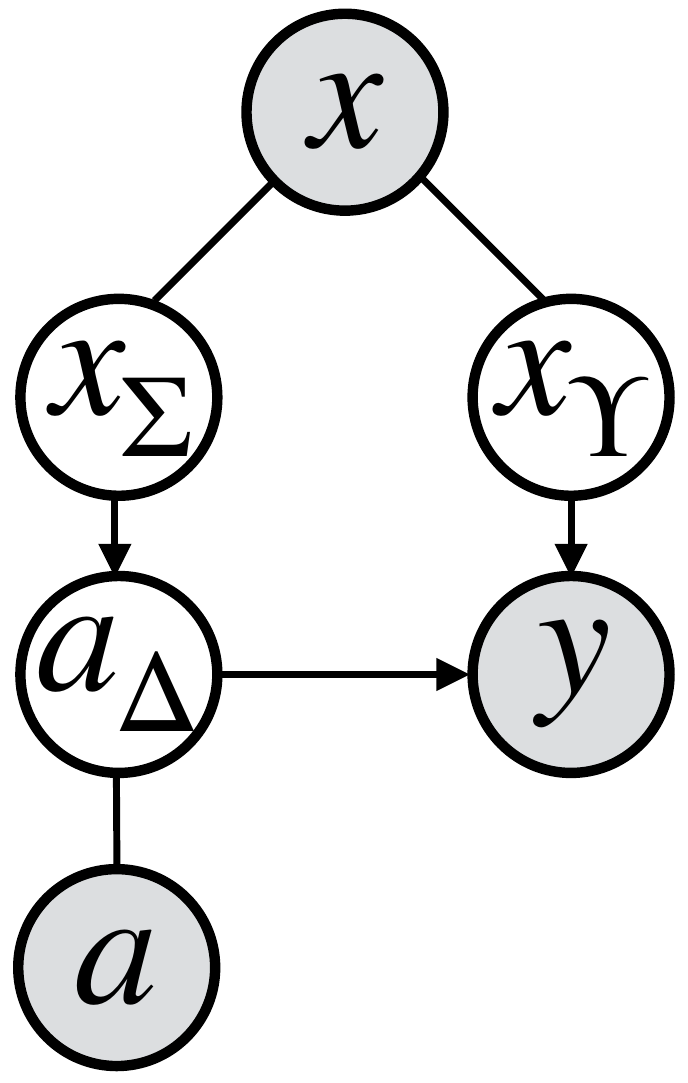}
    \label{fig:graph_synthetic_data_ko}
  }
  \subfigure[Setup-C]{
    \includegraphics[keepaspectratio, width=20mm]{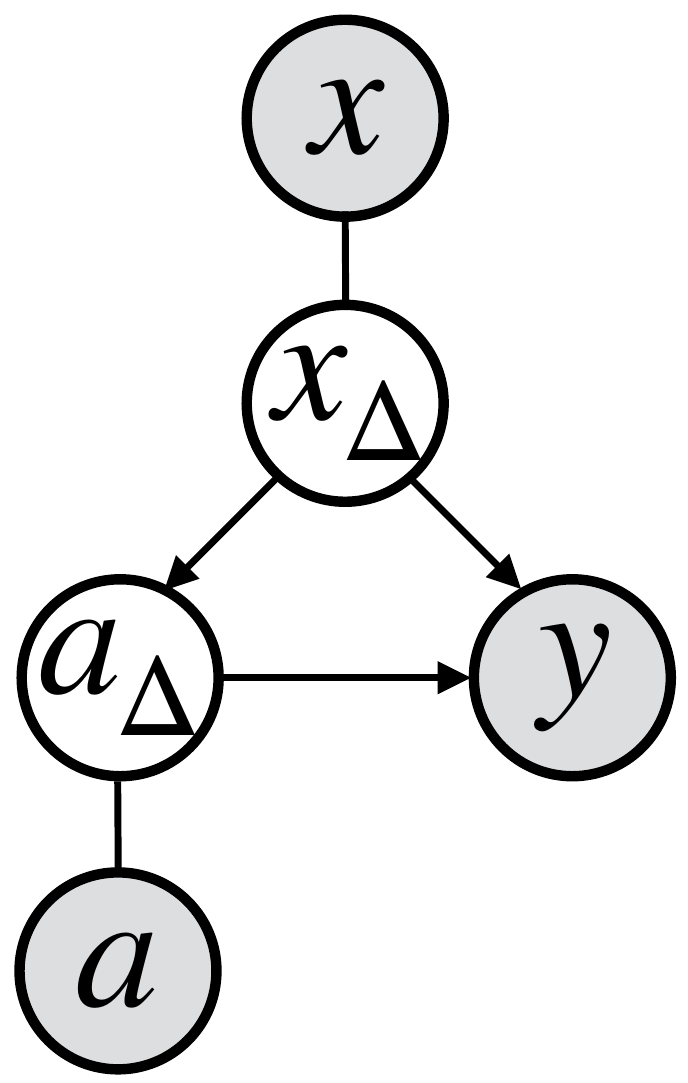}
    \label{fig:graph_synthetic_data_ps}
  }
  \subfigure[Illustration of Linear-C]{
    \includegraphics[keepaspectratio, width=40mm]{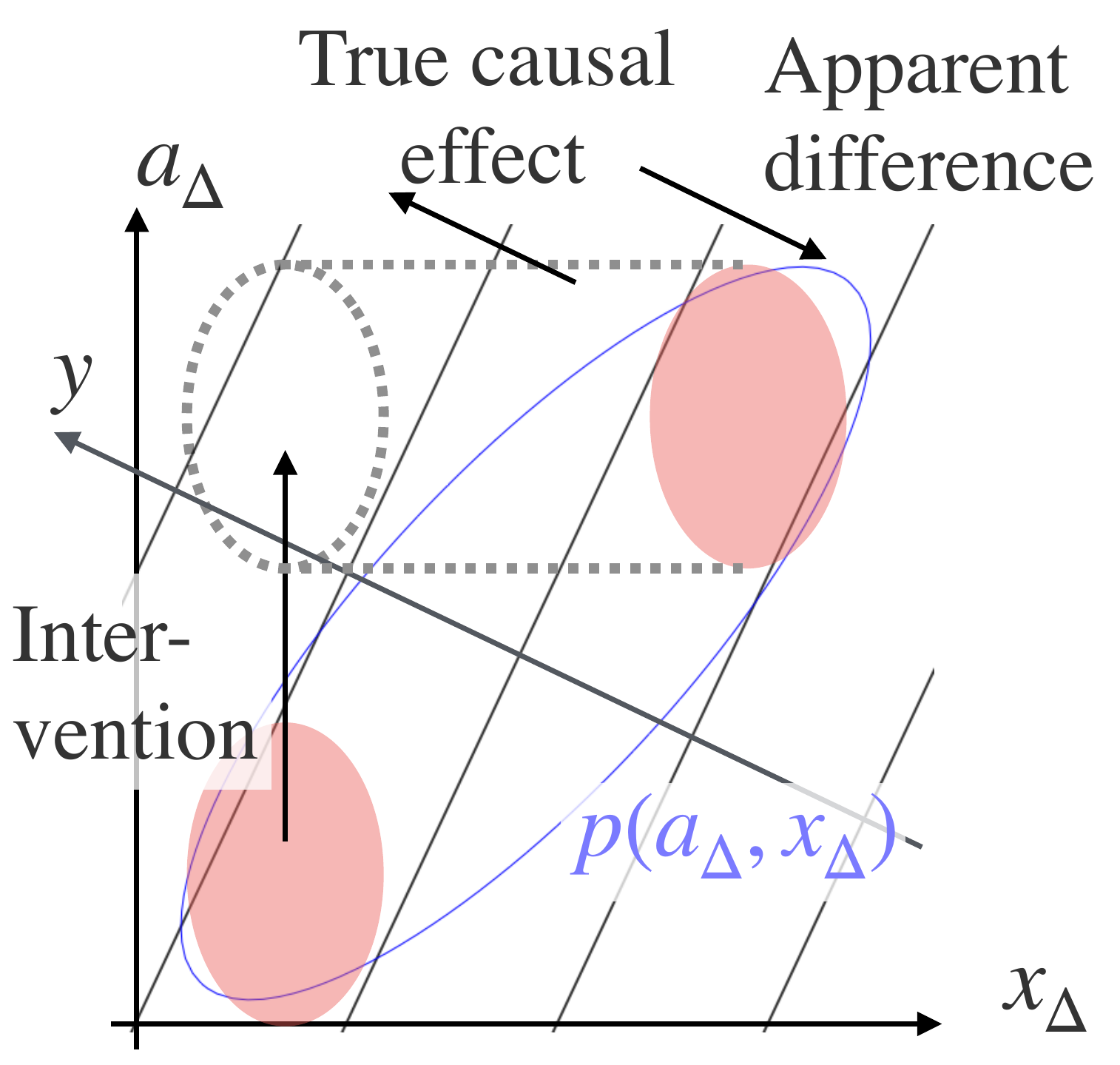}
    \label{fig:illutration_linear_C}
  }
  \label{fig:graph_synthetic_data}
  \caption{\subref{fig:graph_synthetic_data_simple}--\subref{fig:graph_synthetic_data_ps} The data generation models for synthetic experiment. Shaded variables denote the accessible variables in training.
  Non-shaded variables are latent one-dimensional representations of $x$ and $a$.
  \subref{fig:illutration_linear_C} Illustration of how Linear-C setting can mislead learners via sampling bias.
  }
\end{figure}

{\bf Result.}
As can be seen in Table~\ref{tb:cg all synthetic dataset}, our proposed method achieved the best performance or compared favorably under all settings.
The results in MSE is also shown in Appendix~\ref{sec:additional experimental results}.
In the linear settings, OLS achieved on par with the oracle (nmCG=1), since the model class is correctly specified, and its performance dropped under the nonlinear settings.
In real-world situations where the class of the true function is unknown, versatility to the true function classes would be a significant strength of the proposed method.
Under the setting of Linear-C, some of the compared methods performed below the chance rate ($\le 0$).
This is maybe because Linear-C is designed to mislead 
as illustrated in Fig.~\ref{fig:illutration_linear_C}.
Here we would like to mention that Linear-C is not unrealistic, e.g., doctors (the past decision-makers) are likely to give stronger medicines ($a_\Delta$ is large) to more serious patients ($x_\Delta$ is large), and the stronger medicines might appear to rather worsen the patients' health ($y$).
In Linear-C and Quadratic-C, the performance of RMNet was worsened by using $D_\mathrm{IPM}$.
This might be because, in Setup-C, the empirical loss ($L$ in \eqref{eq:objective}) and the regularizer $D_\mathrm{IPM}$ might conflict, i.e., extracting both $x_\Delta$ and $a_\Delta$ is needed for better prediction, which increase $D_\mathrm{IPM}$ due to the bias.
In recent studies \cite{Hassanpour2020Learning, zhang2020learning,pmlr-v89-johansson19a}, it is argued that requiring low IPM is unnecessarily strong, and alternatives for IPM are proposed.
Thus, there is a room for further improvement left in this direction for future work.

\begin{table*}[tb]
  \caption{Synthetic results on normalized mean CG@1 (larger is better and a maximum at one) and its standard error in ten data generations. Best and second-best methods are in bold.}
  \label{tb:cg all synthetic dataset}
  \centering
  \scalebox{0.8}[0.8]{
    \begin{tabular}{|l|rrrrrrr|}
      \toprule
      Method &                 Linear-A &              Linear-B &              Linear-C &            Quadratic-A&   Quadratic-B&         Quadratic-C&               Bilinear \\
      \midrule
      OLS                           &  {\bf 0.99 {\scriptsize $\pm$ .00}} &  {\bf 1.00 {\scriptsize $\pm$ .00}} &  {\bf 1.00 {\scriptsize $\pm$ .00}} &        0.20 {\scriptsize $\pm$ .10} &        0.68 {\scriptsize $\pm$ .11} &        0.80 {\scriptsize $\pm$ .13} &       $-$0.00 {\scriptsize $\pm$ .01} \\
      Random Forest                 &        0.52 {\scriptsize $\pm$ .10} &        0.46 {\scriptsize $\pm$ .06} &       $-$0.89 {\scriptsize $\pm$ .11} &        0.71 {\scriptsize $\pm$ .10} &        0.24 {\scriptsize $\pm$ .03} &        0.90 {\scriptsize $\pm$ .06} &        0.64 {\scriptsize $\pm$ .04} \\
      BART                          &        0.69 {\scriptsize $\pm$ .12} &  {\bf 0.99 {\scriptsize $\pm$ .00}} &       $-$1.03 {\scriptsize $\pm$ .08} &        0.54 {\scriptsize $\pm$ .15} &  {\bf 0.87 {\scriptsize $\pm$ .07}} &  {\bf 0.99 {\scriptsize $\pm$ .00}} &        0.02 {\scriptsize $\pm$ .04} \\
      M-DNN                         &        0.40 {\scriptsize $\pm$ .16} &        0.76 {\scriptsize $\pm$ .09} &        0.07 {\scriptsize $\pm$ .09} &        0.77 {\scriptsize $\pm$ .09} &        0.45 {\scriptsize $\pm$ .14} &        0.62 {\scriptsize $\pm$ .13} &        0.25 {\scriptsize $\pm$ .10} \\
      S-DNN                         &        0.83 {\scriptsize $\pm$ .11} &        0.85 {\scriptsize $\pm$ .08} &        0.64 {\scriptsize $\pm$ .18} &        0.78 {\scriptsize $\pm$ .09} &        0.52 {\scriptsize $\pm$ .08} &        0.70 {\scriptsize $\pm$ .08} &        0.66 {\scriptsize $\pm$ .08} \\
      CFRNet                        &        0.10 {\scriptsize $\pm$ .23} &        0.72 {\scriptsize $\pm$ .16} &        0.06 {\scriptsize $\pm$ .10} &        0.52 {\scriptsize $\pm$ .18} &        0.30 {\scriptsize $\pm$ .12} &        0.53 {\scriptsize $\pm$ .15} &        0.09 {\scriptsize $\pm$ .08} \\
      \midrule
      RMNet                         &  {\bf 0.96 {\scriptsize $\pm$ .01}} &        0.98 {\scriptsize $\pm$ .01} &        0.76 {\scriptsize $\pm$ .07} &  {\bf 0.95 {\scriptsize $\pm$ .02}} &  {\bf 0.87 {\scriptsize $\pm$ .03}} &        0.90 {\scriptsize $\pm$ .05} &  {\bf 0.83 {\scriptsize $\pm$ .02}} \\
      RMNet (w/o MSE)                &        0.94 {\scriptsize $\pm$ .01} &        0.89 {\scriptsize $\pm$ .09} &        0.47 {\scriptsize $\pm$ .09} &        0.93 {\scriptsize $\pm$ .02} &        0.86 {\scriptsize $\pm$ .03} &        0.83 {\scriptsize $\pm$ .05} &        0.75 {\scriptsize $\pm$ .07} \\
      RMNet (w/o ER)                 &        0.90 {\scriptsize $\pm$ .05} &        0.84 {\scriptsize $\pm$ .08} &        0.60 {\scriptsize $\pm$ .13} &        0.88 {\scriptsize $\pm$ .05} &        0.55 {\scriptsize $\pm$ .07} &        0.71 {\scriptsize $\pm$ .08} &        0.56 {\scriptsize $\pm$ .09} \\
      RMNet (w/o $D_{\mathrm{IPM}}$) &        0.91 {\scriptsize $\pm$ .03} &        0.98 {\scriptsize $\pm$ .01} &  {\bf 0.87 {\scriptsize $\pm$ .05}} &  {\bf 0.94 {\scriptsize $\pm$ .01}} &        0.86 {\scriptsize $\pm$ .03} &  {\bf 0.92 {\scriptsize $\pm$ .04}} &  {\bf 0.83 {\scriptsize $\pm$ .02}} \\          

        \bottomrule
      \end{tabular}
  }
\end{table*}

\subsection{Semi-synthetic experiment}
{\bf Dataset (GPU kernel performance).}
For semi-synthetic experiment, we used the SGEMM GPU kernel performance dataset \cite{nugteren2015cltune,ballester2017sobol}, which has 14 feature attributes of GPU kernel parameters and four target attributes of elapsed times in milliseconds for four independent runs for each combination of parameters. 
We used the inverse of the mean elapsed times as the outcome.
Then we had 241.6k instances in total.
By treating some of the feature attributes as action dimensions, we got a {\it complete} dataset, which has all the entries (potential outcomes) in Fig.~\ref{fig:sample_table_combi_single} observed.
Then we composed our semi-synthetic dataset by biased subsampling of only one action $a$ and the corresponding potential outcome $y_a$ for each $x$.
The details of this preprocess can be found in Appendix~\ref{sec: experimental details}.

The sampling policy in the training data was
    $p(a|x,y) \propto \exp(-10 |y - [x^\top,a^\top]^\top w|),$
where $w$ is sampled from $\mathcal{N}(0,1)^{d+m}$. 
This policy reproduced a spurious correlation; that is, a random projection of the feature and the action $[x^\top,a^\top]^\top w$ is likely to have little causal relationship with $y$ but does have a strong correlation due to the sampling policy.
This policy also depends on $y$, which violates the unconfoundedness assumption.
Although, the dataset we used has a low noise level, i.e., $y \simeq g(x,a)$ for some function $g$, and thus $p(a|x,y) \simeq p(a|x,g(x)).$

We split the feature set $\{x_n\}_n$ into 80\% for training, 5\% for validation, and 15\% for testing.
Then, for the training set, only one action $a$ and the corresponding outcome $y$ was taken for each $x$.
The resulting training sample size for each setting of $m$ is listed in Table~\ref{tb:data_size_specification} in Appendix~\ref{sec: experimental details}.
We repeated the training and evaluation process ten times for different splits and samplings of $a$.

{\bf Result.}
As shown in Table~\ref{tb:all gpu dataset}, our proposed method outperformed the others in nmCG@1 in all cases.
In terms of MSE, S-DNN with the same backbone also achieved a high performance,
which demonstrates that the structure in Fig.~\ref{fig:general_structure} efficiently modeled the data.
The performance gains compared to ``w/o ER'' and ``w/o $D_\mathrm{IPM}$'' demonstrate the effectiveness of both of the components proposed in Section~\ref{sec:method}.
The superior performance of RMNet without MSE in the settings of $|\mathcal A| =16$ and $32$ indicates the room for optimizing $\beta$, which we fixed to 0.5.

\begin{table*}[tb]
  \caption{Semi-synthetic results on normalized mean CG@1 and MSE with the standard error in ten different samplings of the training data. Best and second-best methods are in bold.}
  \label{tb:all gpu dataset}
  \centering
  \scalebox{0.75}[0.75]{
    \begin{tabular}{|l|rrrr|rrrr|}
      \toprule
      {} & \multicolumn{4}{c}{Normalized mean CG@1} & \multicolumn{4}{|c|}{MSE} \\
      $|\mathcal A|$ & \multicolumn{1}{c}{8} & \multicolumn{1}{c}{16} & \multicolumn{1}{c}{32} & \multicolumn{1}{c|}{64} & \multicolumn{1}{c}{8} & \multicolumn{1}{c}{16} & \multicolumn{1}{c}{32} & \multicolumn{1}{c|}{64} \\
      Method         &                                 &                                 &                                 &                                 &                         &                         &                         &                         \\
      \midrule
      OLS                           &       $-$0.04 {\scriptsize $\pm$ .15} &       $-$0.08 {\scriptsize $\pm$ .20} &       $-$0.10 {\scriptsize $\pm$ .13} &       $-$0.01 {\scriptsize $\pm$ .10} &        1.12 {\scriptsize $\pm$ .12} &        1.89 {\scriptsize $\pm$ .26} &        1.70 {\scriptsize $\pm$ .26} &        5.86 {\scriptsize $\pm$ 1.10} \\
      Random Forest                 &        0.23 {\scriptsize $\pm$ .08} &        0.33 {\scriptsize $\pm$ .07} &        0.32 {\scriptsize $\pm$ .05} &        0.37 {\scriptsize $\pm$ .05} &        1.03 {\scriptsize $\pm$ .11} &        0.87 {\scriptsize $\pm$ .08} &        0.93 {\scriptsize $\pm$ .09} &        1.07 {\scriptsize $\pm$ .18} \\
      BART                          &        0.00 {\scriptsize $\pm$ .13} &        0.17 {\scriptsize $\pm$ .13} &        0.11 {\scriptsize $\pm$ .10} &        0.04 {\scriptsize $\pm$ .09} &        1.06 {\scriptsize $\pm$ .08} &        1.04 {\scriptsize $\pm$ .08} &        1.19 {\scriptsize $\pm$ .12} &        1.63 {\scriptsize $\pm$ .23} \\
      M-DNN                         &        0.41 {\scriptsize $\pm$ .05} &        0.48 {\scriptsize $\pm$ .06} &        0.31 {\scriptsize $\pm$ .07} &        0.37 {\scriptsize $\pm$ .05} &        0.78 {\scriptsize $\pm$ .05} &        0.84 {\scriptsize $\pm$ .02} &        0.83 {\scriptsize $\pm$ .02} &        0.84 {\scriptsize $\pm$ .02} \\
      S-DNN                         &        0.29 {\scriptsize $\pm$ .09} &        0.26 {\scriptsize $\pm$ .10} &        0.32 {\scriptsize $\pm$ .07} &        0.46 {\scriptsize $\pm$ .05} &  {\bf 0.75 {\scriptsize $\pm$ .12}} &  {\bf 0.60 {\scriptsize $\pm$ .09}} &  {\bf 0.74 {\scriptsize $\pm$ .06}} &        0.74 {\scriptsize $\pm$ .04} \\
      CFRNet                        &        0.50 {\scriptsize $\pm$ .06} &        0.39 {\scriptsize $\pm$ .14} &        0.39 {\scriptsize $\pm$ .10} &        0.35 {\scriptsize $\pm$ .05} &        0.79 {\scriptsize $\pm$ .02} &        0.81 {\scriptsize $\pm$ .02} &        0.87 {\scriptsize $\pm$ .01} &        0.86 {\scriptsize $\pm$ .01} \\
      \midrule
      RMNet                         &  {\bf 0.68 {\scriptsize $\pm$ .00}} &  {\bf 0.60 {\scriptsize $\pm$ .05}} &  {\bf 0.60 {\scriptsize $\pm$ .05}} &  {\bf 0.51 {\scriptsize $\pm$ .05}} &        0.77 {\scriptsize $\pm$ .00} &        0.76 {\scriptsize $\pm$ .09} &        0.84 {\scriptsize $\pm$ .02} &  {\bf 0.73 {\scriptsize $\pm$ .07}} \\
      RMNet (w/o MSE)                &  {\bf 0.68 {\scriptsize $\pm$ .00}} &  {\bf 0.66 {\scriptsize $\pm$ .01}} &  {\bf 0.67 {\scriptsize $\pm$ .01}} &  {\bf 0.50 {\scriptsize $\pm$ .05}} &        0.76 {\scriptsize $\pm$ .00} &        0.75 {\scriptsize $\pm$ .06} &        0.85 {\scriptsize $\pm$ .01} &        0.80 {\scriptsize $\pm$ .08} \\
      RMNet (w/o ER)                 &  {\bf 0.68 {\scriptsize $\pm$ .00}} &        0.45 {\scriptsize $\pm$ .08} &        0.56 {\scriptsize $\pm$ .05} &        0.49 {\scriptsize $\pm$ .05} &        0.77 {\scriptsize $\pm$ .00} &  {\bf 0.67 {\scriptsize $\pm$ .08}} &        0.88 {\scriptsize $\pm$ .02} &        0.75 {\scriptsize $\pm$ .05} \\
      RMNet (w/o $D_{\mathrm{IPM}}$) &        0.33 {\scriptsize $\pm$ .09} &        0.27 {\scriptsize $\pm$ .10} &        0.40 {\scriptsize $\pm$ .07} &        0.48 {\scriptsize $\pm$ .06} &  {\bf 0.72 {\scriptsize $\pm$ .12}} &        0.81 {\scriptsize $\pm$ .18} &  {\bf 0.78 {\scriptsize $\pm$ .08}} &  {\bf 0.71 {\scriptsize $\pm$ .06}} \\

        \bottomrule
      \end{tabular}
  }
\end{table*}

\section{Summary}
In this paper, we have investigated causal inference on a large action space with a focus on the decision-making performance.
We first defined and analyzed the performance in decision-making brought about by a model through a simple prediction-based decision-making policy.
Then we showed that the bound only with the regression accuracy (MSE) gets looser as the action space gets large, which illustrates the difficulty of utilizing causal inference in decision-making in a large action space.
At the same time, however, our bound indicates that minimizing not only the regression loss but also the classification loss leads to a better performance. 
From this viewpoint, our proposed method minimizes both the regression and classification losses, specifically, soft cross-entropy with a teacher label indicating whether an observed outcome is better than the estimated conditional average outcome in the observational distribution under a given feature.
Experiments on synthetic and semi-synthetic datasets, which is designed to have misleading spurious correlations, demonstrated the superior performance of the proposed method with respect to the decision performance and the regression accuracy.

\bibliography{reference}

\bibliographystyle{abbrvnat}

\newpage

\appendix

\section{Proof of Proposition~\ref{thm:relation_cg_mse}}
\label{sec:proof_regret_bound}
\begin{proposition}
  The expected regret will be bounded with uniform MSE in (\ref{eq:unbiased mse}) as
  \begin{align*}
    \mathrm{Regret}@k(f) \le \frac{|\mathcal A|}{k} \sqrt{\mathrm{ER}_k^u(f) \cdot \mathrm{MSE}^u(f)},
  \end{align*}
  where $\mathrm{ER}_k^u(f)$ is the top-$k$ classification error rate, i.e.,
  \begin{align*}
    \mathrm{ER}_k^u(f) := \E_x \left[ \frac{1}{|\mathcal{A}|} \sum_{a \in \mathcal{A}} I\left((\mathrm{rank}(y_{a})\le k) \oplus (\mathrm{rank}(f(x,a))\le k)\right) \right],
  \end{align*}
  where $\oplus$ denotes the logical XOR.
\end{proposition}

\begin{proof}

Here we denote the true and the predicted $i$-th best action by $a^\ast_i$ and $\hat a^\ast_i$, respectively;
i.e., $\mathrm{rank}(y_{a^\ast_i})=\mathrm{rank}(f(x,\hat{a}^\ast_i))=i$.
For all $k \in [|\mathcal{A}|]$, the target-wise regret can be bounded as follows.
\begin{align}
  k \cdot \mathrm{Regret}@k(x) &:= \sum_{i \le k} \left(y_{a^\ast_i} - y_{\hat a^\ast_i} \right) \nonumber \\
  &\le \sum_{i \le k} \left(y_{a^\ast_i} - y_{\hat a^\ast_i} \right) + \sum_{i\leq k} \left(f_{\hat{a}_i^{\ast}}-f_{a_i^{\ast}}\right) \label{eq:proof_insert_f}\\
  &= \sum_{i \le k} \left\{ \left(y_{a^\ast_i} - f_{ a^\ast_i} \right) - \left(y_{\hat a^\ast_i} - f_{\hat a^\ast_i}\right) \right\}, \nonumber
\end{align}
where $f_a = f(x,a)$.
Inequality (\ref{eq:proof_insert_f}) is from the definition of $\hat a^\ast_i$; i.e., $\sum_{i\le k} f_{\hat a^\ast_i}$ is the summation of the top-$k$ $f_a$s out of $\{f_a\}_{a \in \mathcal{A}}$, which must be larger than or equal to the summation of $k$ $f_a$'s that are not necessarily top-$k$, $\sum_{i\le k} f_{a^\ast_i}$.
Let $s= \sum_{i\le k} \left( {\bf 1}_{a^\ast_i} - {\bf 1}_{\hat a^\ast_i}\right),$ where ${\bf 1}_a$ is the one-hot encoding of $a$, and $e$ be the error vector that consists of $e_a=y_a - f_a$. The r.h.s. is bounded as
\begin{align*}
  \text{r.h.s.} &= \left\langle s, e \right\rangle \\
  &\le \| s  \|_2 \cdot \|  e\|_2 \\
  &= |\mathcal A| ~ \sqrt{\mathrm{ER}^u_k(f,x) \cdot \mathrm{MSE}^u_k(f,x)},
\end{align*}
where $\mathrm{ER}^u_k(f,x)$ and $\mathrm{MSE}^u_k(f,x)$ are the target-wise error rate and MSE, respectively.
The inequality comes from the Cauchy–Schwarz inequality. 
By taking the expectation with respect to $x$ and applying Jensen's inequality, we get the proposition.
\end{proof}

\section{Error analysis for representation-based domain adaptation from observational data to the uniform average on action space}
\label{sec:theory}
By performing the representation balancing regularization, our method enjoys better generalization through minimizing the upper bound of the error on the test distribution (under uniform random policy).
We briefly show why minimizing the combination of empirical loss on training and the regularization of distribution (\ref{eq:objective}) results in minimizing the test error.
First, we define the point-wise loss function under a hypothesis $h$ and an extractor $\phi(\cdot,\cdot)$, which defines the representation $\phi = \phi(x,a)$, as
\begin{align*}
    \ell_{h}^a(\phi) := \int_{\mathcal{Y}} L(Y_a, h(\phi))p(Y_a|x) \mathrm{d} Y_a.
\end{align*}
Then, the expected losses for the training (source) and the test distribution (target) are
\begin{align*}  \begin{split}
    \epsilon^s(h):=&\int_{\mathcal{X,A},\Phi} \ell_{h}^a(\phi)p(\phi |x,a)p(x,a) \mathrm{d}\phi \mathrm{d}x\mathrm{d}a, \\
    \epsilon^t(h):=&\int_{\mathcal{X,A},\Phi} \ell_{h}^a(\phi)p(\phi |x,a)p^u(a|x)p(x) \mathrm{d}\phi \mathrm{d}x\mathrm{d}a.
  \end{split}
\end{align*}
We assume there exists $B > 0$ such that $\frac{1}{B} \ell_{h}^a(\phi) \in G$ for the given function space $G$.
Then the integral probability metric $\mathrm{IPM}_G$ is defined for $\phi \in \Phi = \{\phi(x,a)|p(x,a)>0\}$ as
\begin{align*}
  \mathrm{IPM}_G (p_1,p_2) := \sup_{g \in G} \left|\int_{\Phi}g(\phi)(p_1(\phi)-p_2(\phi)) \mathrm{d} \phi\right|.
\end{align*}
The difference between the expected losses under training and test distributions are then bounded as
\begin{align*}
&\epsilon^t(h)-\epsilon^s(h)\\
&=\int_{\Phi} \ell_{h}^a(\phi)\left(p^u(\phi)-p(\phi)\right) \mathrm{d}\phi \\
&=B \int_{\Phi} \frac{1}{B} \ell_{h}^a(\phi)\left(p^u(\phi)-p(\phi)\right) \mathrm{d}\phi \\
&\le B \sup_{g \in G} \left| \int_{\Phi} g(\phi)\left(p^u(\phi)-p(\phi)\right) \mathrm{d}\phi \right|\\
&= B \cdot \mathrm{IPM}_G \left(p(\phi), p^u(\phi) \right).
\end{align*}
For $G$, we use the 1-Lipshitz function class, after which $\mathrm{IPM}_G$ is the Wasserstein distance $D_{\mathrm{wass}}$.
Although $B$ is unknown, the hyperparameter tuning of the regularization strength $\alpha$ in (\ref{eq:objective}) can achieve the tuning of $B$.

\section{Experimental details}
\label{sec: experimental details}
{\bf Synthetic data generation process.}
Our synthetic datasets are built as follows.
\begin{itemize}
  \item[1]Sample $x \sim \mathcal{N}(0,1)^d,$ where $d=5$.
  \item[2] Sample $a \in \{0,1\}^m$, where $m=5$, from $p(a|x) \propto \exp(10\left|x_\Sigma - a_\Sigma\right|),$ where $x_\Sigma$ and $a_\Sigma$ are the following.
  \begin{itemize}
    \item[2-1] In settings other than Setup-B, $x_\Sigma=x_\Delta=w_x^\top x,$ where $w_x \sim \mathcal N(0,1/d)^d$.
    \item[2-2] In Setup-B, $x_\Sigma = x_1$, i.e., only the first dimension in $x$ is used to bias $a$.
    \item[2-3] $a_\Sigma=w_a^\top a,$ where $w_a \sim \mathcal N(0,1/m)^m$.
  \end{itemize}
  \item[3] Calculate the expected oucome $y_a=f(x,a),$ where we examine three types of functions $f$, namely, Linear, Quadratic, and Bilinear. In the Linear and Quadratic types, $f(x,a) = f(x_\Upsilon, a_\Upsilon),$ where $x_\Upsilon$ and $a_\Upsilon$ are one-dimensional representations of $x$ and $a$, respectively.
  \begin{itemize}
    \item[3-1] In Setup-B, $x_\Upsilon = w_{x, 2:d}^\top x_{2:d},$ where $x_{2:d}$ denotes all dimensions other than the first dimension ($x_\Sigma$).
    \item[3-2] In settings other than Setup-B, $x_\Upsilon=x_\Sigma (=:x_\Delta)$.
    \item[3-3] In Setup-A, $a_\Upsilon = w_a^{\prime \top} a$ where $w_a^\prime \sim \mathcal N(0,1/m)^m.$
    \item[3-4] In settings other than Setup-A, $a_\Upsilon=a_\Sigma (=:a_\Delta)$.
    \item[3-5] In the Linear setting, $f(x_\Upsilon,a_\Upsilon)=a_\Upsilon - 2 x_\Upsilon.$
    \item[3-6] In the Quadratic setting, $f(x_\Upsilon,a_\Upsilon)=a_\Upsilon^2 - 2 x_\Upsilon.$
    \item[3-7] In the Bilinear setting, $f(x,a) = x^\top W a,$ where $W \sim \mathcal N(0,1/(dm))^{(d,m)}$
  \end{itemize}
  \item[4] Sample the observed outcome $y \sim \mathcal N(y_a|0.1).$
\end{itemize}

{\bf Details of semi-synthetic data}
We transformed the target attributes of elapsed times into the average speed as the outcome, i.e., $y=\frac{4}{\sum z_i}$, where $\{z_i\}_{1:4}$ are the original elapsed times.
Then we standardized $y$ and the features.
Each feature can take binary values or up to four different powers of two values.
Out of 1,327k total parameter combinations, only 241.6k feasible combinations are recorded.
We split these original feature dimensions into $a$ and $x$ as follows.
The dimension of the action space $m$ ranged from three to six, and the 8th, 11th, 12th, 13th, 14th, and 3rd dimensions are regarded as $a$ from the head in order (e.g., for $m=3$, the 8th, 11th, and 12th dimensions in the original features are regarded as $a$).
This split was for maximizing the overlap of $\mathcal{A}(x)$ among $\mathcal{X}$.

{\bf Other DNN parameters.}
The detailed parameters we used for DNN-based methods (S-DNN, M-DNN, CFRNet, and proposed) were as follows.
The backbone DNN structure had four layers for representation extraction and three layers for hypothesis with the width of 64 for the middle layers and the width of 10 for the representation $\phi$.
The batch size was 64, but only for CFRNet, it was 512 for the need to approximate the distributions for each action. 
The strength of our used L2 regularizer was $10^{-4}$.
We used Adam \cite{kingma2015adam} for the optimizer with the learning rate of $10^{-4}$.

\begin{table}[tb]
  \caption{Training sample size for each setting.}
  \label{tb:data_size_specification}
  \centering
  \scalebox{0.9}[0.9]{
    \begin{tabular}{|rrr|}
      \hline
      $m$ & $|\mathcal{A}|$ & \multicolumn{1}{c|}{ $N_{\mathrm{tr}} $}\\
      \hline
      3	&8	&24,160 \\
      4	&16	&12,080 \\
      5	&32	&6,040 \\
      6	&64	&3,591 \\
        \hline
      \end{tabular}
  }
\end{table}

\section{Derivation of Eq.~\ref{eq:ranking-to-clf}}
\label{sec:der-ranking-to-clf}
Recall
\begin{align*}
  \mathrm{ER}_k^u(f) &:= 
  \E_x\Bigg[\frac{1}{|\mathcal{A}|} \sum_{a \in \mathcal{A}} I\left((\mathrm{rank}(y_{a})\le k) \oplus (\mathrm{rank}(f(x,a))\le k)\right) \Bigg] .
\end{align*}
Since
\begin{align*}
  I\left((\mathrm{rank}(y_{a})\le k) \oplus (\mathrm{rank}(f(x,a))\le k)\right)
  &=I\left( (y_{a_k^\ast} \le y_a) \oplus (f(x,\hat a_k^\ast) \le f(x,a))\right) \\
  &=I\left( (y_{a_k^\ast} \le y_a) \oplus (y_{a_k^\ast} \le f(x,a)-f(x,\hat a_k^\ast) + y_{a_k^\ast})\right) \\
  &=\ell_{0-1}(y_a - y_{a_k^\ast}, f^\prime(x,a) - y_{a_k^\ast}),
\end{align*}
we have
\begin{align*}
  \mathrm{ER}_k^u(f) &=
  \E_{x}\Bigg[
  \frac{1}{|\mathcal{A}|} \sum_{a \in \mathcal{A}}
  \ell_{0-1}(y_a - y_{a_k^\ast}, f^\prime(x,a) - y_{a_k^\ast}) \Bigg] .
\end{align*}
Here $f^\prime$ satisfies the condition $f^\prime(x, \hat a_k^\ast) = y_{a_k^\ast}$, i.e., the $k$-th largest prediction of $f^\prime$ should be equal to $ y_{a_k^\ast}$.
Although, since $y_{a_k^\ast}$ is unobservable, we relax the optimization of $f^\prime$ in the function space that satisfies the condition into the optimization in the general function space.
Assuming that our function space includes the optimal function $f^\ast$ that minimizes 

\section{Additional experimental results}
\label{sec:additional experimental results}

\begin{table*}[tb]
  \caption{Synthetic results on MSE and its standard error in ten data generations. Best and second-best methods are in bold.}
  \label{tb:mse all synthetic dataset}
  \centering
  \scalebox{0.79}[0.79]{
    \begin{tabular}{|l|rrrrrrr|}
      \toprule
      Method &                 Linear-A &              Linear-B &              Linear-C &            Quadratic-A&   Quadratic-B&         Quadratic-C&               Bilinear \\
      \midrule
      OLS                           &  {\bf 0.01 {\scriptsize $\pm$ 0.00}} &  {\bf 0.01 {\scriptsize $\pm$ 0.00}} &  {\bf 0.01 {\scriptsize $\pm$ 0.00}} &        2.70 {\scriptsize $\pm$ 0.62} &   {\bf 9.91 {\scriptsize $\pm$ 4.53}} &  {\bf 10.89 {\scriptsize $\pm$ 4.16}} &        0.28 {\scriptsize $\pm$ 0.03} \\
      Random Forest                 &       20.29 {\scriptsize $\pm$ 5.34} &        3.19 {\scriptsize $\pm$ 0.54} &       17.48 {\scriptsize $\pm$ 5.14} &       16.83 {\scriptsize $\pm$ 5.58} &        12.59 {\scriptsize $\pm$ 5.18} &        19.80 {\scriptsize $\pm$ 4.87} &        0.24 {\scriptsize $\pm$ 0.03} \\
      BART                          &       14.59 {\scriptsize $\pm$ 3.80} &        0.67 {\scriptsize $\pm$ 0.23} &       14.30 {\scriptsize $\pm$ 3.33} &       14.62 {\scriptsize $\pm$ 3.82} &        11.58 {\scriptsize $\pm$ 4.51} &        18.64 {\scriptsize $\pm$ 3.77} &        0.50 {\scriptsize $\pm$ 0.10} \\
      M-DNN                         &       10.70 {\scriptsize $\pm$ 2.23} &        3.70 {\scriptsize $\pm$ 2.07} &       12.05 {\scriptsize $\pm$ 2.25} &       12.79 {\scriptsize $\pm$ 2.26} &        16.44 {\scriptsize $\pm$ 5.99} &        19.20 {\scriptsize $\pm$ 3.66} &        0.36 {\scriptsize $\pm$ 0.07} \\
      S-DNN                         &        0.64 {\scriptsize $\pm$ 0.23} &        1.04 {\scriptsize $\pm$ 0.46} &        2.61 {\scriptsize $\pm$ 1.21} &        5.18 {\scriptsize $\pm$ 2.25} &        16.75 {\scriptsize $\pm$ 5.87} &        16.43 {\scriptsize $\pm$ 3.12} &        0.13 {\scriptsize $\pm$ 0.03} \\
      CFRNet                        &       10.09 {\scriptsize $\pm$ 2.22} &        5.01 {\scriptsize $\pm$ 2.14} &       12.64 {\scriptsize $\pm$ 2.42} &       10.04 {\scriptsize $\pm$ 2.18} &        13.26 {\scriptsize $\pm$ 4.09} &        16.25 {\scriptsize $\pm$ 3.58} &        0.40 {\scriptsize $\pm$ 0.07} \\
      \midrule
      RMNet                         &  {\bf 0.30 {\scriptsize $\pm$ 0.08}} &  {\bf 0.48 {\scriptsize $\pm$ 0.09}} &        2.45 {\scriptsize $\pm$ 0.46} &  {\bf 1.75 {\scriptsize $\pm$ 0.96}} &        10.59 {\scriptsize $\pm$ 4.69} &        13.82 {\scriptsize $\pm$ 3.87} &  {\bf 0.10 {\scriptsize $\pm$ 0.01}} \\
      RMNet (w/o MSE)                &        0.46 {\scriptsize $\pm$ 0.14} &        2.31 {\scriptsize $\pm$ 1.46} &        3.14 {\scriptsize $\pm$ 0.79} &        2.07 {\scriptsize $\pm$ 1.00} &  {\bf 10.55 {\scriptsize $\pm$ 4.69}} &  {\bf 13.52 {\scriptsize $\pm$ 4.17}} &        0.13 {\scriptsize $\pm$ 0.04} \\
      RMNet (w/o ER)                 &        0.46 {\scriptsize $\pm$ 0.14} &        1.65 {\scriptsize $\pm$ 0.62} &        3.71 {\scriptsize $\pm$ 1.10} &        3.55 {\scriptsize $\pm$ 1.82} &        16.68 {\scriptsize $\pm$ 5.87} &        16.33 {\scriptsize $\pm$ 3.10} &        0.19 {\scriptsize $\pm$ 0.03} \\
      RMNet (w/o $D_{\mathrm{IPM}}$) &        0.50 {\scriptsize $\pm$ 0.14} &        0.54 {\scriptsize $\pm$ 0.10} &  {\bf 1.29 {\scriptsize $\pm$ 0.27}} &  {\bf 1.71 {\scriptsize $\pm$ 0.69}} &        11.87 {\scriptsize $\pm$ 4.38} &        14.43 {\scriptsize $\pm$ 3.87} &  {\bf 0.08 {\scriptsize $\pm$ 0.01}} \\
      \bottomrule
      \end{tabular}
  }
\end{table*}

{\bf Elapsed times compared to CFR}  
Figure \ref{fig:elapsed_time} shows the comparison in training time between the proposed method and CFRNet.
For CFRNet, the elapsed time grew when the size of the action space $|\mathcal{A}|$ became large.
The main reason for this is the calculation of distance between the representation distributions for each pair of actions $\sum_{a \neq a^\prime} D_{\mathrm{IPM}}\left(p_a(\phi), p_{a^\prime}(\phi)\right)$ in Fig.~\ref{fig:structure_cfr}.
The decrease of the elapsed time for RMNet is mainly due to the sample sizes shown in Table~\ref{tb:data_size_specification}.

\begin{figure}[tb]
\centering
 \includegraphics[keepaspectratio,width=70mm]{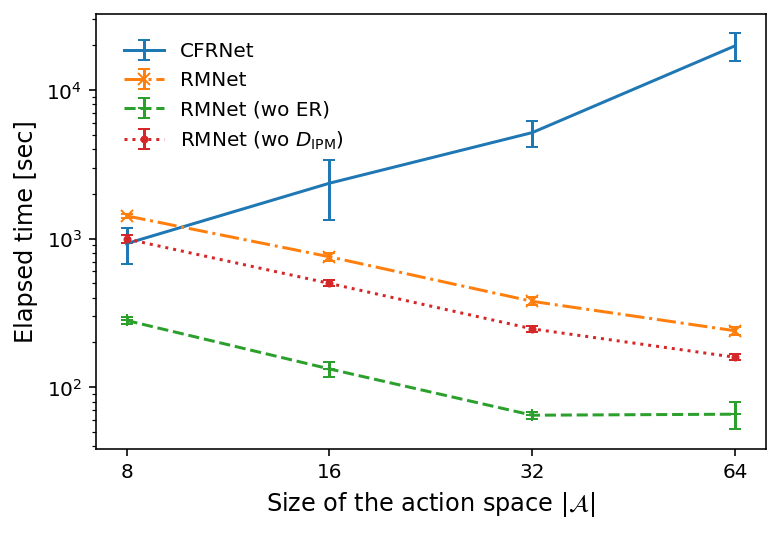}
 \caption{Elapsed time for training. Error bars indicate standard deviation.}
\label{fig:elapsed_time}
\end{figure}

\begin{figure}[tb]
  \centering
   \includegraphics[keepaspectratio,width=70mm]{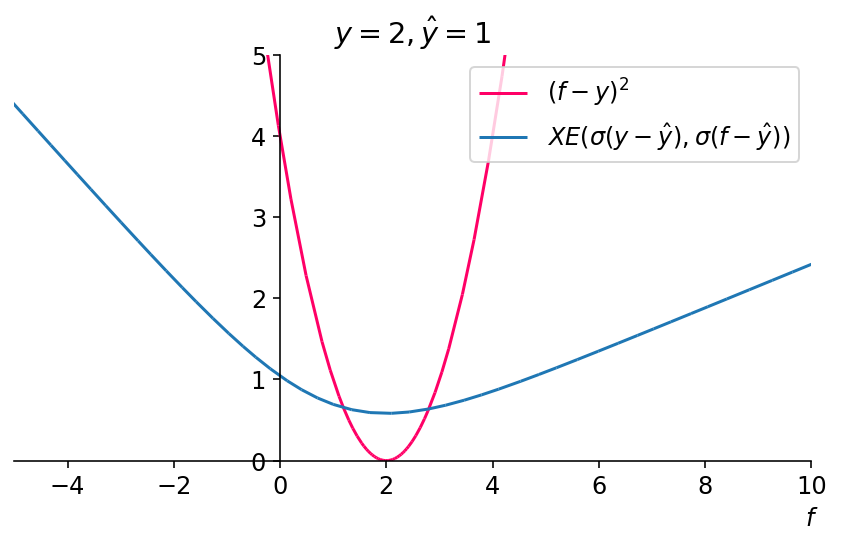}
   \caption{Comparison of MSE and our proposed soft cross-entropy (denoted as $\mathrm{\it XE}$) for a training instance. 
   $\hat y$ denotes the estimation of the conditional expectation on the observational distribution $\E[y|x]$.
   The soft cross-entropy also takes the minimum value when $f(x,a)=y$. 
   The asymmetry of the loss works as follows.
   The actual outcome was larger than the estimated conditional expectation, i.e., $y-\hat y=1 \ge 0$, in this case, which means that the action $a$ of this instance performed ``relatively well'' compared to the estimated average performance in the observational data under $x$.
   Therefore, predicting $a$ as a better action than average ($f(x,a) \ge \hat y$) is regarded as ``successfully classified'' and penalized less than the failed case ($f(x,a) \le \hat y$) for the same regression error.}
  \label{fig:loss_shape}
  \end{figure}

\end{document}